\documentclass[english]{IEEEtran}
\IEEEoverridecommandlockouts    

\usepackage{microtype}
\usepackage{graphicx}
\usepackage{subcaption}
\usepackage{booktabs} 
\usepackage{enumitem} 
\usepackage{hyperref}

\usepackage{amsmath}
\usepackage{amssymb}
\usepackage{mathtools}
\usepackage{amsthm}
\usepackage{algorithm}
\usepackage{algorithmic}

\usepackage[capitalize,noabbrev]{cleveref}

\usepackage[textsize=tiny]{todonotes}

\theoremstyle{plain}

\newtheorem{proposition}{Proposition}
\newtheorem{lemma}{Lemma}

\theoremstyle{definition}
\newtheorem{definition}{Definition}
\newtheorem{assumption}{Assumption}
\theoremstyle{remark}
\newtheorem{remark}{Remark}

\title{\LARGE \bf
Conformalized Signal Temporal Logic Inference under Covariate Shift
}

\author{
Yixuan Wang$^{1}$, Danyang Li$^{2}$, Matthew Cleaveland$^{3}$, Roberto Tron$^{2}$, Mingyu Cai$^{1}$
\vspace{-0.6em}
\thanks{ $^{1}$Mingyu Cai and Yixuan Wang are with Mechanical Engineering, University of California, Riverside, CA, 92521, USA.  
        {\tt\small \{mingyu.cai, ywang1457\}@ucr.edu}
$^{2}$Danyang Li and Roberto Tron are with Mechanical Engineering, Boston University.
        {\tt\small \{danyangl, tron\}@bu.edu}   
$^{3}$Matthew Cleaveland is with MIT Lincoln
Laboratory, Lexington, MA, 02421, USA
        {\tt\small \{ matthew.cleaveland@ll\}@mit.edu}     
        }
\thanks{DISTRIBUTION STATEMENT A. Approved for public release. Distribution is unlimited.
This material is based upon work supported by the Department of the Army under Air Force Contract No. FA8702-15-D-0001 or FA8702-25-D-B002. Any opinions, findings, conclusions or recommendations expressed in this material are those of the author(s) and do not necessarily reflect the views of the Department of the Army.
© 2026 Massachusetts Institute of Technology.
Delivered to the U.S. Government with Unlimited Rights, as defined in DFARS Part 252.227-7013 or 7014 (Feb 2014). Notwithstanding any copyright notice, U.S. Government rights in this work are defined by DFARS 252.227-7013 or DFARS 252.227-7014 as detailed above. Use of this work other than as specifically authorized by the U.S. Government may violate any copyrights that exist in this work.
}%

}

\begin{document}

\maketitle
\thispagestyle{empty}
\pagestyle{empty}




\begin{abstract}
Signal Temporal Logic (STL) inference learns interpretable logical rules for temporal behaviors in dynamical systems. To ensure the correctness of learned STL formulas, recent approaches have incorporated conformal prediction as a statistical tool for uncertainty quantification. However, most existing methods rely on the assumption that calibration and testing data are identically distributed and exchangeable, an assumption that is frequently violated in real-world settings.
This paper proposes a conformalized STL inference framework that explicitly addresses covariate shift between training and deployment trajectories dataset.
From a technical standpoint, the approach first employs a template-free, differentiable STL inference method to learn an initial model, and subsequently refines it using a limited deployment side dataset to promote distribution alignment. To provide validity guarantees under distribution shift, the framework estimates the likelihood ratio between training and deployment distributions and integrates it into an STL-robustness-based weighted conformal prediction scheme.
Experimental results on trajectory datasets demonstrate that the proposed framework preserves the interpretability of STL formulas while significantly improving symbolic learning reliability at deployment time.
The project page can be found: \url{https://sites.google.com/ucr.edu/confrtlics?usp=sharing}.
\end{abstract}

\begin{IEEEkeywords}
Formal Methods, Conformal Prediction, Signal Temporal Logic, Temporal Logic Inference, 
\end{IEEEkeywords}

\section{Introduction}

Interpretable decision models are particularly valuable in control and robotics, where learned decision rules are often used in safety-critical settings and must therefore be inspected and verified. Signal Temporal Logic (STL) provides an interpretable formal language for describing temporal behaviors of dynamical systems. It specifies temporal properties through human-readable logical formulas and equips them with a quantitative robustness semantics to measure the degree of satisfaction or violation. These properties make STL a useful foundation for learning interpretable classifiers over trajectories.
A substantial body of prior work has studied STL inference from data. Early methods typically relied on fixed logical templates or small template families and optimized real-valued parameters using robustness-based objectives~\cite{asarin2011parametric,hoxha2018mining}. More recent approaches have expanded the space of learnable specifications through logic-based learning from demonstrations, and differentiable neural-symbolic frameworks that embed temporal and Boolean operators into trainable computation graphs~\cite{bartocci2022survey,bombara2021offline,yoo2017rich,leung2023backpropagation,li2024tlinetdifferentiableneuralnetwork, li2023learning, yan2019swarm}. These developments improve scalability and flexibility while preserving the interpretability of the learned specifications. In particular, recent differentiable STL learning approaches optimize both formula structure and parameters directly from trajectory data, providing a practical route to compact and expressive symbolic classifiers.

However, interpretability alone is not sufficient for deployment. In many control and robotics applications, a learned classifier is not only expected to be accurate, but also to quantify the reliability of the decisions. A symbolic STL formula may be human-readable, yet its predictions can still be poorly calibrated or statistically unreliable when learned from finite data. This motivates the need for statistical correctness guarantees in addition to interpretable formula learning. Conformal prediction (CP) provides a principled framework for this purpose by converting a real-valued score into a calibrated decision rule with finite-sample coverage guarantees under exchangeability~\cite{vovk2005algorithmic,angelopoulos2021gentle}. In practice, split CP uses a held-out calibration set and empirical quantiles nonconformity scores, which are then used as thresholds for constructing calibrated decision rules~\cite{papadopoulos2008inductive,lei2018distribution,romano2019conformalized}. CP has also been used in robotics and control to provide calibrated decision rules and uncertainty-aware safety wrappers for learning-based systems~\cite{dixit2023adaptive, liang2025time, lindemann2023safe}. In previous works~\cite{soroka2024learning, Li2025ConformalPF}, CP was integrated with differentiable STL inference e.g., TLINet~\cite{li2024tlinetdifferentiableneuralnetwork} and STLcg~\cite{leung2023backpropagation}, to quantify prediction uncertainty while preserving interpretability.

A further challenge is that standard CP relies on exchangeability, meaning that the calibration data and the future test data are assumed to follow the same distributional mechanism. STL formula are typically learned from trajectories collected under nominal or controlled conditions, whereas deployment trajectories may differ because of environmental variation, operational changes, or data-collection bias. Such covariate shift can substantially degrade deployment-time reliability even when the conditional labeling rule remains unchanged. Existing STL inference methods primarily optimize empirical performance on source-side data and do not explicitly address calibration under distribution mismatch between source-side and deployment-side distributions. Prior conformalized STL inference~\cite{ soroka2024learning,Li2025ConformalPF}, does not explicitly address mismatch among training, calibration, and deployment data. Weighted CP provides one statistical mechanism for handling such mismatch by reweighting calibration samples according to their deployment relevance~\cite{tibshirani2019conformal, dixit2023adaptive}. As a result, we posit that in the STL setting, reliability under distribution shift should be addressed jointly at two levels: the learned formula should adapt toward deployment-relevant regions, and the final decision threshold should be calibrated in a distribution-aware manner. Our framework is more broadly compatible with differentiable STL inference methods. In the sequel, we formulate the method using binary trajectory classification as a concrete instantiation.

Motivated by this challenge, we study conformalized STL inference under covariate shift. Our approach starts from a differentiable STL inference model and refines the learned formula using an additional dataset that is more representative of deployment conditions. The refinement stage adapts the formula toward the test-time distribution through data reweighting~\cite{shimodaira2000improving,loftsgaarden1965nonparametric} and robustness-distribution regularization~\cite{hamza2003jensen}. We then perform weighted CP for the resulting robustness-based decision rule, so that calibration better reflects the test-time distribution. In this way, the framework jointly addresses formula adaptation and statistical calibration under distribution shift. Contributions are as follows:

\begin{itemize}[leftmargin=*, noitemsep, topsep=0pt]
    \item We develop a covariate-shift-aware incremental learning scheme for differentiable STL inference that adapts a pretrained STL formula through STL robustness based distributional alignment, while maintaining both classification performance and interpretability.
    \item To provide statistical correctness under distribution mismatch, we extend conformalized STL inference to the covariate shift setting using weighted CP~\cite{tibshirani2019conformal} with a robustness-based nonconformity score tailored to learned STL formulas.
    \item To improve the stability of shift alignment during learning, we introduce a distribution-level termination criterion that ineffective learning iterations.
    \item Experiments on trajectory datasets demonstrate that the proposed method achieves well-calibrated inference coverage under covariate shift and reduces empirical miscoverage.
\end{itemize}

\section{Preliminaries}
\label{section2}
We work with trajectories of the form $x=(x_0,\dots,x_{T})$, where $x_t\in\mathbb{R}^d$ denotes the system state at time $t$.


\subsection{Signal Temporal Logic Inference}

Signal Temporal Logic (STL) is a formal language for describing temporal and spatial properties of time-series data.
Atomic predicates are taken to be linear inequalities of the form $\mu(x_t) \equiv a^\top x_t \ge b, a \in \mathbb{R}^d,\;b \in\mathbb{R}$, which test whether a linear constraint on the state is satisfied at a given time. STL formula are generated from atomic predicates according to the grammar
\[
  \varphi ::= \mu
  \;\big|\;
  \neg \varphi
  \;\big|\;
  \varphi_1 \land \varphi_2
  \;\big|\;
  \varphi_1 \lor \varphi_2
  \;\big|\;
  \Diamond_{[t_1,t_2]} \varphi
  \;\big|\;
  \Box_{[t_1,t_2]} \varphi
\]
where $\neg$, $\land$ and $\lor$ denote Boolean negation, conjunction and disjunction, $\Diamond_{t_1,t_2}$ and $\Box_{[t_1,t_2]}$ are the temporal "eventually" and “always” operators over the interval $[t_1,t_2] \subseteq \{0,\dots,T\}$.

The standard quantitative semantics of STL associates to each formula
$\varphi$, trajectory $X$, and time $t$ a real-valued robustness score
$\rho_\varphi(X,t) \in \mathbb{R}$ that measures the degree of satisfaction. We denote $\rho_\varphi(X) := \rho_\varphi(X, 0)$. Positive values indicate satisfaction, negative values indicate violation, and larger absolute values correspond to larger robustness margins. For atomic predicates, Boolean operators, and temporal operators, the robust semantics are defined recursively as follows:
\begin{align}
\rho_{\mu}(X,t) &= a^\top x_t - b \label{eq:atomic} \\
\rho_{\neg \varphi}(X,t) &= - \rho_{\varphi}(X,t) \label{eq:neg} \\
\rho_{\varphi_1 \wedge \varphi_2}(X,t) &= 
\min\!\big(\rho_{\varphi_1}(X,t), \rho_{\varphi_2}(X,t)\big)\label{eq:and} \\
\rho_{\varphi_1 \vee \varphi_2}(X,t) &= 
\max\!\big(\rho_{\varphi_1}(X,t), \rho_{\varphi_2}(X,t)\big) \label{eq:or} \\
\rho_{\Diamond_{[t_1,t_2]}\varphi}(X,t) &=
\max_{\tau \in [t_1,t_2]} \rho_{\varphi}(X,t+\tau) \label{eq:eventually} \\
\rho_{\Box_{[t_1,t_2]}\varphi}(X,t) &=
\min_{\tau \in [t_1,t_2]} \rho_{\varphi}(X,t+\tau) \label{eq:always}
\end{align}


STL formulas can be used as binary classifiers for trajectories through the sign of the robustness score. Given a labeled trajectory $(X,Y)$ with $Y\in\{-1,+1\}$, define
\begin{equation}
\hat{Y}(X)=
\begin{cases}
+1, & \rho_\varphi(X) > 0,\\
-1, & \text{otherwise}
\end{cases}
\end{equation}
Thus, the product $Y\rho_\varphi(X)$ is positive when the robustness sign agrees with the true label and negative otherwise.

\paragraph{Misclassification rate (MCR)}
For an evaluation dataset $D=\{(X_i,Y_i)\}_{i=1}^{n}$, the misclassification rate (MCR) induced by $\varphi$ is defined as
\[
\mathrm{MCR}(D;\varphi)
=
\frac{1}{n}\sum_{i=1}^{n}\mathbf{1}\!\left[\hat{Y}(X_i)\neq Y_i\right]
\]
Minimizing the MCR corresponds to learning an STL formula whose robustness sign aligns with the ground-truth labels.

\paragraph{Differentiable STL learning (TLINet)}

TLINet~\cite{li2024tlinetdifferentiableneuralnetwork} is adopted as a differentiable STL learner for optimizing STL formulas from data. Since the 0-1 misclassification loss $\mathbf{1}[Y\rho_{\varphi_\theta}(X)\le 0]$ is non-differentiable, optimization is performed using a smooth margin-based surrogate, namely the logistic loss, as described in Sec~\ref{sec:tlinet}.

\subsection{Conformal Prediction}
\label{sec:nonconformity score}
Conformal Prediction (CP)~\cite{angelopoulos2021gentle} is a distribution-free procedure that augments the output of a fixed predictive model with prediction regions that enjoy finite-sample coverage guarantees. In split CP, calibration is performed on a held-out dataset using a nonconformity score, which measures how inconsistent a labeled sample is with the model's decision rule. The dataset is partitioned into a training set $D_{\mathrm{train}}$, a calibration set $D_{\mathrm{cal}}=\{(X_i,Y_i)\}_{i=1}^{n_{\mathrm{cal}}}$, and a test set $D_{\mathrm{test}}$, where $D_{\mathrm{cal}}$ and a new test sample are exchangeable. In our shift-aware setting, the calibration data are not assumed exchangeable with the test data. Instead, weighted CP is used to approximate deployment-relevant nonconformity score quantiles under covariate shift.

Let $A:\mathcal{X}\times\mathcal{Y}\to\mathbb{R}$ denote a measurable nonconformity score function,
where $\mathcal{X}$ is the trajectory space and $\mathcal{Y}$ is the label space.
The nonconformity scores on the calibration set are computed as
\begin{equation}
s_i = A\bigl(X_i,Y_i), \quad i = 1, \ldots, n_{\mathrm{cal}}.
\end{equation}

For a target miscoverage level $\alpha \in (0,1)$, the conformal threshold is defined
as the $(1-\alpha)$ quantile of the multiset $\{s_1, \ldots, s_{n_{\mathrm{cal}}}\}$.
Specifically, let $k = \lceil (n_{\mathrm{cal}} + 1)(1 - \alpha) \rceil$ and define $T_{\mathrm{CP}} := s_{(k)}$, where $s_{(k)}$ denotes the
$k$-th smallest value. The resulting CP set for a test input $X_{test}\in\mathcal{X}$ is
\[\mathcal{C}(X_{test}):=\{Y_{test}\in\mathcal{Y}:A(X_{test},Y_{test})\le T_{\mathrm{CP}}\}\]
Under the exchangeability assumption, the CP guarantee yields
\begin{equation}
\mathbb{P}\!\left(Y_{\mathrm{test}}\in \mathcal{C}(X_{\mathrm{test}})\right)\ge 1-\alpha
\end{equation}


\section{Problem Formulation}
\label{sec:3}
We consider the problem of binary classification of system trajectories using STL. Let $\varphi_\theta$ denote an STL formula learned by a differentiable STL inference method, where $\theta$ collects the learnable formula parameters, such as predicate and temporal parameters. The induced robustness score $\rho_{\varphi_\theta}(X)$ is used to define the binary classification.

\paragraph{Data distributions and covariate shift}
Following the distribution-shift setting described in the introduction, let $P_{\mathrm{train}}$ denote the nominal training distribution and $P_{\mathrm{dep}}$ denote the deployment distribution. We consider a covariate shift setting in which the conditional labeling rule remains invariant while the marginal distribution over trajectories changes.
\begin{assumption}[Covariate Shift under STL Semantics]
\label{sec:assump1}
We assume
\[
P_{\mathrm{train}}(Y \mid X) = P_{\mathrm{dep}}(Y \mid X),
\qquad
P_{\mathrm{train}}(X) \neq P_{\mathrm{dep}}(X)
\]
This assumption is natural in our setting because labels are defined by STL satisfaction, which depends only on the trajectory and the specification, and is therefore invariant to operating conditions.
\end{assumption}

\paragraph{Data partitioning}
We consider three mutually disjoint datasets. A core training set $D_{\mathrm{train}} \sim P_{\mathrm{train}}$ is used to learn an initial STL formula and define task semantics. Because it is collected under nominal or conservative conditions, it may not adequately cover the range of trajectories encountered at deployment. We assume access to a limited deployment-side dataset $D_{\mathrm{dep}}$ that captures behaviors underrepresented in $D_{\mathrm{train}}$. Although $D_{\mathrm{dep}}$ is only a partial and potentially biased sample from $P_{\mathrm{dep}}$, it provides information about regions likely to arise at test time and is incorporated only through distribution-aware objectives, without altering the labeling rule. Finally, a separate calibration set $D_{\mathrm{cal}} \sim P_{\mathrm{train}}$ is reserved exclusively for weighted conformal calibration of the robustness-based decision rule~\cite{angelopoulos2021gentle,lindemann2023safe}. Since $D_{\mathrm{cal}}$ is drawn from $P_{\mathrm{train}}$ whereas deployment-time test trajectories follow $P_{\mathrm{dep}}$, standard calibration is no longer distribution-matched. We therefore use importance weighting to make calibration better reflect the target deployment distribution.

\paragraph{Problem Formulation}


Given $D_{\mathrm{train}}$, $D_{\mathrm{dep}}$, and $D_{\mathrm{cal}}$ as defined above, the objective is to learn $\varphi_\theta$. For a $\theta$, the corresponding  prediction set is induced by the calibration rule and is denoted by $C_\theta(\cdot)$. The resulting STL decision rule should achieve both low misclassification and reliable coverage under the deployment distribution:
\[
\begin{aligned}
\min_{\theta}\quad
& \mathrm{MCR}_{P_{\mathrm{dep}}}(\varphi_\theta)
:= \mathbb{E}_{(X,Y)\sim P_{\mathrm{dep}}}\!\left[\mathbf{1}\!\left[Y\rho_{\varphi_\theta}(X)\le 0\right]\right] \\
\text{s.t.}\quad
& \mathbb{P}_{(X,Y)\sim P_{\mathrm{dep}}}\!\left(Y\in C_{\theta}(X)\right)\ge 1-\alpha
\end{aligned}
\]

\begin{remark}
\textit{Challenges and Motivations:}
If the covariate shift is ignored, an STL formula learned only from $D_{\mathrm{train}}$ may overfit regions of the trajectory space that are well represented during training but less relevant at deployment, leading to degraded classification performance and miscalibration under $P_{\mathrm{dep}}$. This can be effective when $D_{\mathrm{dep}}$ is sufficiently large and representative, but in practice the available deployment-side data is often limited and distributionally mismatched relative to $D_{\mathrm{train}}$, making naive retraining statistically unstable.
\end{remark}

Our approach instead uses the STL formula learned from $D_{\mathrm{train}}$ as a warm start and incorporates $D_{\mathrm{dep}}$ through distribution-aware objectives that steer the learned robustness scores toward regions more relevant to deployment. To assess how informative the deployment-side data is, we monitor an effective sample size diagnostic derived from density-ratio weights. This quantity reflects the concentration of the weights and indicates whether further refinement additional refinement is likely to materially improve deployment alignment.

\begin{figure}[t]
    \centering
    \includegraphics[width=\columnwidth]{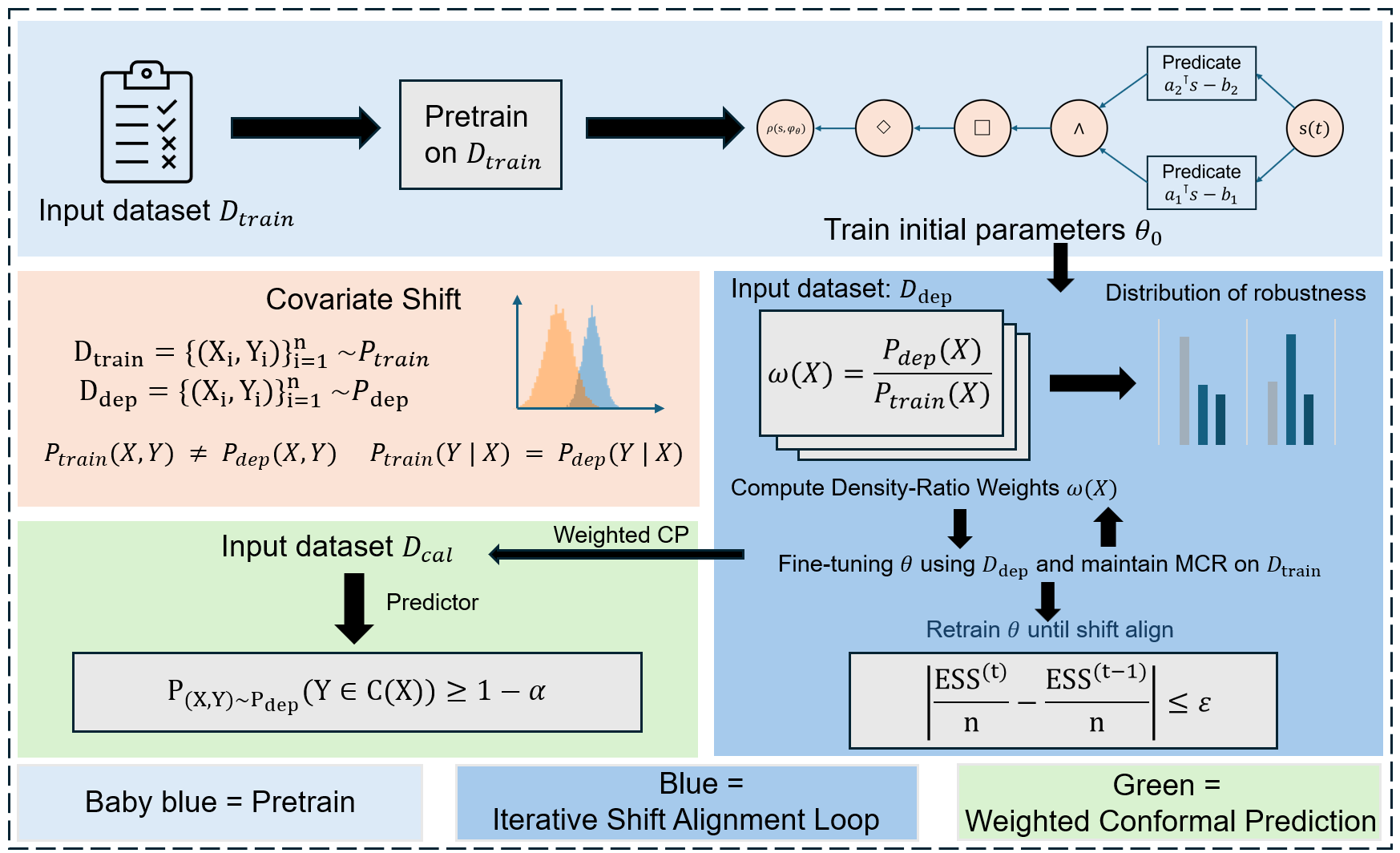}
    \caption{Pipeline of proposed shift-aware conformal STL framework under covariate shift. The framework consists of four sequential stages integrating STL formula learning, distribution alignment, and weighted conformal calibration.}
    \label{fig:framework diagram}
    \vspace{-1.5em}
\end{figure}

\section{Conformalized Active STL Inference under Covariate Shift \label{sec:TLICP}}

In Section~\ref{sec:tlinet}, we present a shift-aware STL framework that learns a formula from training data and adapts it using deployment-proxy data. Section~\ref{sec:CP} establishes correctness guarantees under distribution shift. The overall pipline is summerized in Fig.~\ref{fig:framework diagram}

\vspace{-0.5em}
\subsection{STL Inference Adaptation}
\label{sec:tlinet}

The framework builds on TLINet~\cite{li2024tlinetdifferentiableneuralnetwork} as a differentiable, template-free backbone for learning STL formulas from labeled trajectories, enabling gradient-based optimization and yielding interpretable, parameterized STL specifications.
Training is performed by minimizing a margin-based surrogate loss applied to the signed robustness value $Y\rho_{\varphi_\theta}(X)$. Following TLINet, the logistic loss is used:
\[
\ell(Y\rho_{\varphi_\theta}(X))
=
\log(1+\exp(-Y\rho_{\varphi_\theta}(X)))
\]

\noindent Given a core training dataset $D_{\mathrm{train}}=\{(X_i,Y_i)\}_{i=1}^{n_c}$, the nominal STL formula is obtained by minimizing the empirical surrogate loss
\[
\mathcal{L}_{\mathrm{train}}(\theta)
=
\frac{1}{n_c}
\sum_{(X,Y)\in D_{\mathrm{train}}}
\ell(Y\rho_{\varphi_\theta}(X))
\]

Let $\theta_0$ denote a minimizer of $\mathcal{L}_{\mathrm{train}}$, and $\varphi_{\theta_0}$ the corresponding STL formula learned from $D_{\mathrm{train}}$. The parameter $\theta_0$ defines the nominal STL formula $\varphi_{\theta_0}$ learned from the training dataset $D_{\mathrm{train}}$. This formula serves as a warm-start
initialization for the subsequent shift-aware adaptation stage, and provides the semantic structure of the STL specification.
The adaptation stage further refines them using deployment-side data $D_{\mathrm{dep}}$ toward deployment-relevant trajectory regions through the distribution-aware objectives introduced below. 



\paragraph{Density-Ratio Weighting on the Alignment Set}

As defined in Section~\ref{sec:3}, the learning setting assumes a covariate shift between training and deployment trajectories. For a trajectory $X \in \mathcal{X}$, the ideal importance weight is the density ratio that exactly reweights expectations from the training distribution to the deployment distribution:

\begin{equation}
\omega^\star(X) := \frac{p_{\mathrm{dep}}(X)}{p_{\mathrm{train}}(X)}
\end{equation}
where $p_{\mathrm{train}}$ and $p_{\mathrm{dep}}$ denote the underlying marginal densities of trajectories under training and deployment distributions, respectively. 

Direct estimation of $\omega^\star(X)$ is generally difficult without strong parametric assumptions. We therefore approximate it using a $k$-nearest-neighbor (kNN) density-ratio estimator in an embedding space~\cite{zhao2022analysis}, which provides a simple nonparametric approximation based on local sample density. Other density-ratio estimators, such as kernel mean matching (KMM), KLIEP, and uLSIF, could also be used~\cite{sugiyama2012density}. The intuition is that, in a $p$-dimensional embedding space, the local sample density around $X$ is approximately inversely proportional to the volume of its kNN neighborhood, and hence to $r_k(X;S)^p$. Therefore, the ratio of kNN radii provides a local approximation to the density ratio between the deployment and training distributions. To instantiate this estimator, let $f:\mathcal{X}\rightarrow\mathbb{R}^p$ denote a fixed embedding. For a reference set $\mathcal{S}$, we define the kNN radius of $X$ as
\[
r_k(X;\mathcal{S}) = \| f(X) - f(X^{(k)}) \|_2
\]
where $X^{(k)} \in \mathcal{S}$ is the $k$-th nearest neighbor of $X$. The unnormalized kNN density-ratio weight is then given by
\[
\hat{\omega}(X) =
\left( \frac{r_k(X; D_{\mathrm{train}})}{r_k(X; D_{\mathrm{dep}})} \right)^p
\]

To improve numerical stability, we normalize the raw weights to have unit empirical mean over $D_{\mathrm{dep}}$ and clip them to a fixed range:
\begin{equation}
\tilde{\omega}(X)
=
\frac{\hat{\omega}(X)}
{\frac{1}{|D_{\mathrm{dep}}|}
\sum_{X' \in D_{\mathrm{dep}}}\hat{\omega}(X')}
\label{eq:weight_norm}
\end{equation}

\begin{equation}
\omega(X)
=
\min\!\left\{
\max\!\left\{\tilde{\omega}(X),\,\omega_{\min}\right\}
\,\omega_{\max}
\right\}
\label{eq:weight_clip}
\end{equation}

Hence, $\omega(X)$ is used as a practical approximation to the ideal density-ratio weight $\omega^\star(X)$. Although not exact, it captures relative distributional differences between $P_{train}$ and $P_{dep}$ and emphasizes trajectories more relevant to deployment.

\begin{lemma}
\label{lem:iw}
Under Assumption~\ref{sec:assump1},  assume that $P_{\mathrm{dep}}$ is absolutely continuous with respect to $P_{\mathrm{train}}$ over $X$. Then, with ideal density-ratio weight $\omega^\star(X)$, for any measurable function $h:\mathcal{X}\times\{-1,+1\}\to\mathbb{R}$ with finite expectation, we have
\[
\mathbb{E}_{(X,Y)\sim P_{\mathrm{dep}}}[h(X,Y)]
=
\mathbb{E}_{(X,Y)\sim P_{\mathrm{train}}}[\omega^\star(X)\,h(X,Y)]
\]

\end{lemma}

\begin{proof}
By the law of total expectation under $P_{\mathrm{dep}}$,
\[
\begin{aligned}
\mathbb{E}_{P_{\mathrm{dep}}}[h(X,Y)]
&=
\int_{\mathcal{X}} \sum_{y\in\{-1,+1\}}
h(x,y)\,p_{\mathrm{dep}}(y\mid x)\,p_{\mathrm{dep}}(x)\,dx
\end{aligned}
\]
Using $p_{\mathrm{dep}}(x)=\omega^\star(x)\,p_{\mathrm{train}}(x)$, we obtain
\[
\begin{aligned}
\mathbb{E}_{P_{\mathrm{dep}}}[h(X,Y)]
&=\int_{\mathcal{X}} \sum_{y\in\{-1,+1\}}
h(x,y)\,\omega^\star(x)\,\\
&\qquad\qquad
p_{\mathrm{train}}(y\mid x)\,
p_{\mathrm{train}}(x)\,dx \\
&=
\mathbb{E}_{P_{\mathrm{train}}}[\omega^\star(X)\,h(X,Y)]
\end{aligned}
\]
\end{proof}


\begin{proposition}
\label{pro:pro1}
Under covariate shift, consider the expected misclassification rate under the deployment distribution $P_{\mathrm{dep}}$. For any fixed STL formula $\varphi_\theta$, the misclassification rate under the deployment distribution can be written as
\begin{multline}
\mathbb{E}_{(X,Y)\sim P_{\mathrm{dep}}}
\bigl[\mathbf{1}\{Y\rho_{\varphi_\theta}(X)\le 0\}\bigr]\nonumber\\
= \mathbb{E}_{(X,Y)\sim P_{\mathrm{train}}}
\bigl[\omega^\star(X)\,
\mathbf{1}\{Y\rho_{\varphi_\theta}(X)\le 0\}\bigr]
\end{multline}
In practice, $\omega^\star(X)$ is replaced by its kNN-based approximation $\omega(X)$.
\end{proposition}

\begin{proof}
This proposition is a special case of Lemma~\ref{lem:iw} with $h(X,Y)=\mathbf{1}\{Y\rho_{\varphi_0}(X)\le 0\}$.
\end{proof}

Proposition~\ref{pro:pro1} motivates a deployment-time learning objective that combines the nominal training loss on $D_{train}$ with a weighted loss on
$D_{dep}$.
\[
\begin{aligned}
\mathcal{L}_{train}(\theta)
= {} &
\frac{1}{|D_{\mathrm{train}}|}
\sum_{(X,Y)\in D_{\mathrm{train}}}
\ell\!\left(Y\,\rho_{\varphi_\theta}(X)\right) \\
&+
\frac{1}{|D_{\mathrm{dep}}|}
\sum_{(X,Y)\in D_{\mathrm{dep}}}
\omega(X)\,
\ell\!\left(Y\,\rho_{\varphi_\theta}(X)\right)
\end{aligned}
\]
The first term preserves the nominal STL semantics learned from $D_{\mathrm{train}}$, while the second term emphasizes deployment-relevant trajectories through density-ratio weighting. However, these two terms alone are not sufficient. In practice, the estimated density-ratio weights can be highly non-uniform, so direct optimization of the weighted term may lead to unstable updates and may distort the robustness structure learned from $D_{\mathrm{train}}$~\cite{tibshirani2019conformal}. To stabilize refinement and further adapt the learned STL rule at the distribution level, we add a robustness-distribution regularizer based on the Jensen--R\'enyi divergence (JRD)~\cite{giraldo2013information}.

\paragraph{Regularization of Robustness-Value Distributions via Jensen--R\'enyi Divergence}
The JRD is adopted to compare the empirical robustness distributions induced on $D_{train}$ and $D_{dep}$. As a symmetric divergence between probability measures, it provides a distribution-level adaptation term that encourages the learned robustness values to remain aligned across the source and deployment-side data. Other distributional discrepancies, e.g., Jensen--Shannon divergence~\cite{shui2022novel}, could also be used.

\begin{definition}
For a given parameter $\theta$, define the empirical robustness distributions induced by the training set and the deployment set as
\begin{equation}
\begin{aligned}
\hat P_{\mathrm{train}}^\theta
&:=
\frac{1}{|D_{\mathrm{train}}|}
\sum_{(X,Y)\in D_{\mathrm{train}}}
\delta_{\rho_{\varphi_\theta}(X)}
\\
\hat P_{\mathrm{dep}}^\theta
&:=
\frac{1}{|D_{\mathrm{dep}}|}
\sum_{X\in D_{\mathrm{dep}}}
\delta_{\rho_{\varphi_\theta}(X)}
\end{aligned}
\end{equation}
where $\delta_{\rho_{\varphi_\theta}(X)}$ denotes the unit point mass located at the robustness value $\rho_{\varphi_\theta}(X)$.
\end{definition}

The following proposition shows that this regularizer yields an explicit bound on the resulting robustness-distribution mismatch.

\begin{proposition}
Consider the regularized training objective
\[
\mathcal{L}(\theta)
=
\mathcal{L}_{\mathrm{train}}(\theta)
+
\lambda_{\mathrm{JRD}}\,
\mathrm{JRD}\!\left(
\hat P^{\theta}_{\mathrm{train}},
\hat P^{\theta}_{\mathrm{dep}}
\right)
\]
where $\lambda_{\mathrm{JRD}}>0$ is a regularization parameter controlling the strength of the robustness-distribution alignment penalty. Let $\theta^\star$ be a minimizer of $\mathcal{L}$, and let  $\theta_{\mathrm{train}}$ be a minimizer of the unregularized training loss $\mathcal{L}_{\mathrm{train}}$. Then
\[
\begin{aligned}
\mathrm{JRD}\!\left(\hat P_{\mathrm{train}}^{\theta^\star},\hat P_{\mathrm{dep}}^{\theta^\star}\right)
\le\;&
\frac{
\mathcal{L}_{\mathrm{train}}(\theta_{\mathrm{train}})
-
\mathcal{L}_{\mathrm{train}}(\theta^\star)
}{\lambda_{\mathrm{JRD}}}
\\
&\quad+
\mathrm{JRD}\!\left(
\hat P_{\mathrm{train}}^{\theta_{\mathrm{train}}},
\hat P_{\mathrm{dep}}^{\theta_{\mathrm{train}}}
\right)
\end{aligned}
\]
\end{proposition}

\begin{proof}
By optimality of $\theta^\star$, we have $\mathcal{L}(\theta^\star)\le \mathcal{L}(\theta_{\mathrm{train}})$. Rearranging terms yields
\[
\begin{aligned}
\lambda_{\mathrm{JRD}}\,
\mathrm{JRD}\!\left(
\hat P^{\theta^\star}_{\mathrm{train}},
\hat P^{\theta^\star}_{\mathrm{dep}}
\right)
\le\;&
\mathcal{L}_{\mathrm{train}}(\theta_{\mathrm{train}})
-
\mathcal{L}_{\mathrm{train}}(\theta^\star)
\\
&\quad+
\lambda_{\mathrm{JRD}}\,
\mathrm{JRD}\!\left(
\hat P^{\theta_{\mathrm{train}}}_{\mathrm{train}},
\hat P^{\theta_{\mathrm{train}}}_{\mathrm{dep}}
\right)
\end{aligned}
\]
which implies the stated bound.
\end{proof}

Although the practical density-ratio weights are clipped to improve numerical stability, clipping alone does not directly control distribution-level drift in the induced robustness values. We therefore refine the pretrained STL formula by minimizing the following unified shift-aware objective:
\begin{equation}
\label{eq:loss}
\begin{aligned}
\mathcal{L}(\theta)
= {} &
\frac{1}{|D_{\mathrm{train}}|}
\sum_{(X,Y)\in D_{\mathrm{train}}}
\ell\!\left(Y\,\rho_{\varphi_\theta}(X)\right) \\
&+
\frac{1}{|D_{\mathrm{dep}}|}
\sum_{(X,Y)\in D_{\mathrm{dep}}}
\omega(X)\,
\ell\!\left(Y\,\rho_{\varphi_\theta}(X)\right) \\
&+
\lambda_{\mathrm{JRD}}\,\mathrm{JRD}\!\left(\hat P^{\theta}_{\mathrm{train}},\hat P^{\theta}_{\mathrm{dep}}\right)
\end{aligned}
\end{equation}
where $\omega(X)$ denotes the density-ratio weight, and $\lambda_{\mathrm{JRD}}>0$ is a regularization parameter.

\vspace{-1.0em}
\subsection{Weighted Conformalized STL Inference}\label{sec:CP}

This section aims to provide correctness guarantees using CP for the learned STL formulas. The general nonconformity score in Section~\ref{sec:nonconformity score} is specialized to the robustness-based form for binary classification.
\begin{equation}
\label{eq:nonconformity score}
    A(X,Y) = S_{\theta}(X,Y) := -Y\,\rho_{\varphi_\theta}(X),
\end{equation}

where larger values of $S_\theta$ indicates stronger disagreement with label $Y$.
This choice is simple and compatible with our prior formulation~\cite{Li2025ConformalPF}. More importantly, it yields a unified decision statistic: the same signed robustness margin governs STL training, binary classification, and CP.


Under covariate shift, calibration and deployment-time test samples are generally not exchangeable. We therefore adopt weighted CP~\cite{tibshirani2019conformal}, where the calibration nonconformity scores are reweighted according to the density-ratio weights introduced in Section~\ref{sec:tlinet}. The nonconformativity scores on $D_{\mathrm{cal}}$ are defined as
\begin{equation}
s_i := S_{\theta}(X_i,Y_i), \qquad i=1,\dots,n_{\mathrm{cal}}
\end{equation}


The weighted empirical cumulative distribution of the calibration nonconformity scores is defined as
\begin{equation}
\hat{F}_w(t)
:=
\frac{\sum_{i=1}^{n_{\cal}}\omega(X) \, \mathbf{1}\{s_i \le t\}}
{\sum_{i=1}^{n_{\cal}}\omega(X)}
\label{eq:weighted_ecdf}
\end{equation}
The weighted conformal threshold is then taken as the $(1-\alpha)$-quantile of this weighted empirical distribution:
\begin{equation}
T_{\mathrm{wcp}}
:=
\inf \left\{
t \in \mathbb{R} : \hat{F}_w(t) \ge 1-\alpha
\right\}
\label{eq:weighted_threshold}
\end{equation}
Accordingly, the weighted CP set is defined by
\begin{equation}
C(X)
=
\{Y \in \{-1,+1\} : S_{\theta}(X,Y) \le T_{\mathrm{wcp}}\}.
\label{eq:weighted_prediction_set}
\end{equation}
Under Assumption~\ref{sec:assump1}, and using the ideal density-ratio weight $\omega^\star(X)$ defined in Section~\ref{sec:tlinet}, the prediction set $C(X)$ is a direct specialization of the weighted conformal construction under covariate shift in~\cite{tibshirani2019conformal} with coverage guarantees, i.e.,
\begin{equation}
\mathbb{P}_{(X,Y)\sim P_{\mathrm{dep}}}\!\left(Y\in C(X)\right)\ge 1-\alpha
\end{equation}

In practice, however, highly non-uniform weights may reduce the effective number of weighted samples and increase the variability of weighted empirical estimates. We can quantify this effect by adopting the effective sample size (ESS)~\cite{6279353}.


\[
\mathrm{ESS}(D)
=
\frac{\left(\sum_{X\in D}\omega(X)\right)^2}
{\sum_{X\in D}\omega(X)^2}
\]
with the normalized form $\mathrm{ESS}(D)/|D|\in(0,1]$. The following result formalizes the relation between weight concentration and the conditional variability of weighted empirical estimates.


\begin{proposition}
\label{lem:ess_variance}
Let $\{Z_i\}_{i=1}^{n}$ be i.i.d.\ from a distribution $Q$ and let $\{w_i\}_{i=1}^{n}$ be nonnegative weights with normalized weights $\bar w_i := w_i / \sum_{j=1}^{n} w_j$. For any measurable $g$ with $g(z)\in[-1,1]$ for all $z$, define
\[
\hat\mu_w(g) := \sum_{i=1}^{n} \bar w_i\, g(Z_i)
\]
Then, conditional on the weights $(w_1,\ldots,w_n)$,
\[
\mathrm{Var}\!\bigl[\hat\mu_w(g)\mid w_{1:n}\bigr]
\;\le\;
\sum_{i=1}^{n}\bar w_i^2
\;=\;
\frac{1}{\mathrm{ESS}},
\]
\end{proposition}
This bound shows that the conditional variability of a weighted empirical estimate is controlled by \(1/\mathrm{ESS}\). In particular, smaller ESS yields a larger upper bound on the conditional variance, indicating reduced stability when the weights are highly concentrated.

\begin{proof}
The quantity $\hat\mu_w(g)$ is a weighted average of independent variables $g(Z_1),\ldots,g(Z_n)$, each bounded in $[-1,1]$. Conditional on the weights, the variance is therefore bounded by the sum of squared normalized weights. Since $|g|\le 1$, $\mathrm{Var}[g(Z_i)]\le 1$. Independence yields
\[
\mathrm{Var}\!\left[\sum_{i=1}^{n} \bar w_i g(Z_i)\right]
=
\sum_{i=1}^{n}\bar w_i^2\, \mathrm{Var}[g(Z_i)]
\le
\sum_{i=1}^{n}\bar w_i^2
\]
Finally,
\[
\sum_{i=1}^{n}\bar w_i^2
=
\frac{\sum_{i=1}^{n} w_i^2}{(\sum_{i=1}^{n}w_i)^2}
=
\frac{1}{\mathrm{ESS}}
\]
\end{proof}
\vspace{-1.0em}

As a result, we can use ESS as a learning termination criterion of balancing the overfitting and maintaining the effective number of calibration samples for the distributional alignment during shift iterative refinement procedure. 
In particular, let $\mathrm{ESS}^{(t)}$ denote the effective sample size at iteration $t$. learning prcoess in Section~\ref{sec:tlinet} is terminated when
\[
\left|
\frac{\mathrm{ESS}^{(t)}}{n}
-
\frac{\mathrm{ESS}^{(t-1)}}{n}
\right|
\le \varepsilon
\]
where $n$ is the number of calibration samples and $\varepsilon > 0$ is a tolerance parameter. This criterion indicates that further updates no longer substantially change the effective distribution alignments.

\section{Experimental Results}
We evaluate the proposed conformalized STL inference framework on three trajectory datasets with distinct distributional characteristics under covariate shift. All experiments were conducted on a Linux workstation equipped with an Intel Core i9-13900KF CPU (24 cores, 32 threads). The objectives of the experiments are threefold:
\begin{itemize}[leftmargin=*, noitemsep, topsep=0pt]
\item To assess the classification performance of the refined STL classifier compared to nominal TLINet training;
\item To evaluate the behavior of CP under distribution shift; 
\item To examine the efficiency-coverage tradeoff between standard CP and weighted CP.
\end{itemize}

\subsection{Evaluation Metrics}
For a target miscoverage level $\alpha$, the conformal predictor aims to achieve a coverage level of $1-\alpha$. In experiments we report the empirical coverage, defined as
\begin{equation}
\text{Coverage} =
\frac{1}{n}\sum_{i=1}^{n}
\mathbf{1}\{Y_i \in C(X_i)\}
\end{equation}
which estimates the probability that the prediction set contains the true label under the test distribution.

To quantify the efficiency of the prediction sets, we report the average prediction set size, referred to as inefficiency:
\begin{equation}
\text{Inefficiency} =
\frac{1}{n}\sum_{i=1}^{n} |C(X_i)|
\end{equation}
This metric measures the average size of prediction sets, where smaller values correspond to more compact predictions. In the binary classification setting considered here, the prediction set $C(X)$ is defined in Section~\ref{sec:nonconformity score} has cardinality in $\{0,1,2\}$. A singleton set corresponds to a confident prediction, a two-label set indicates ambiguity, and an empty set indicates that neither label satisfies the acceptance condition.

\subsection{Datasets}
\textbf{Dataset 1:} \emph{Naval Surveillance:}
The Naval Surveillance dataset consists of labeled time-series trajectories derived from a marine propulsion system benchmark. Each trajectory is a multivariate signal $X \in \mathbb{R}^{d \times T}$, and the task is binary classification of normal versus anomalous operational behavior. 

\textbf{Dataset 2:} \emph{Place a block into a basket without violating the safety constraint.}
This task requires the robot to transport a target block into a designated basket region while keeping the trajectory within the basket boundary marked by tape.


\begin{figure}[t]
\centering
\begin{subfigure}[b]{0.48\linewidth}
    \centering
    \includegraphics[width=\linewidth]{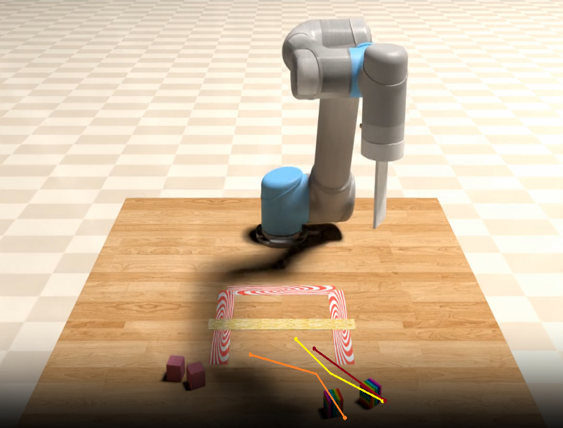}
    \caption{VIMA simulated environment}
\end{subfigure}
\hfill
\begin{subfigure}[b]{0.48\linewidth}
    \centering
    \includegraphics[width=\linewidth]{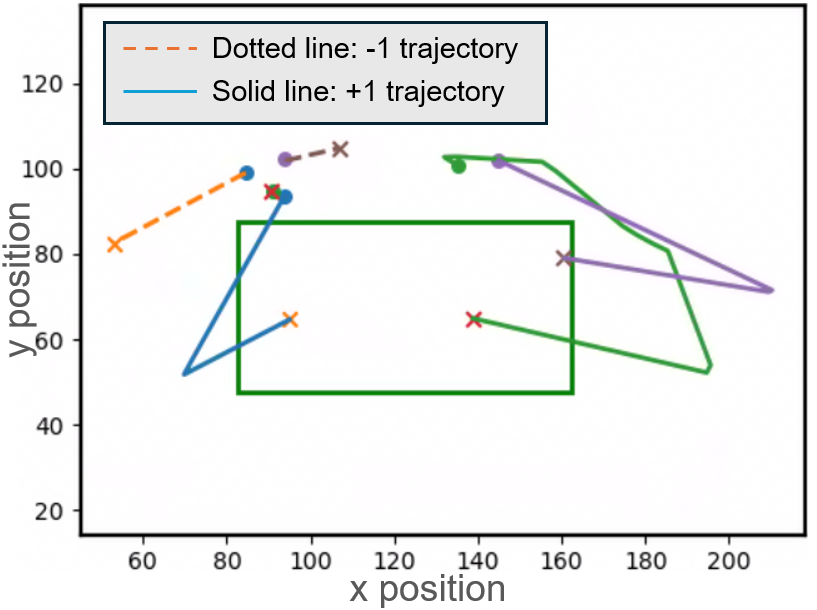}
    \caption{Vima Dataset}
\end{subfigure}
\hfill
\begin{subfigure}[b]{0.48\linewidth}
    \centering
    \includegraphics[width=\linewidth]{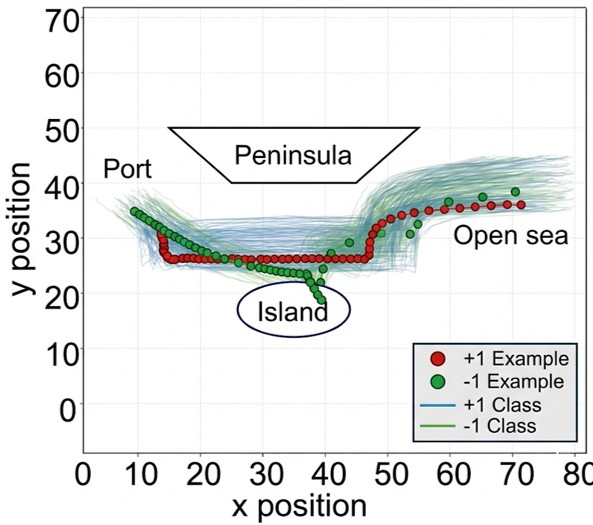}
    \caption{Naval Dataset}
\end{subfigure}
\hfill
\begin{subfigure}[b]{0.48\linewidth}
    \centering
    \includegraphics[width=\linewidth]{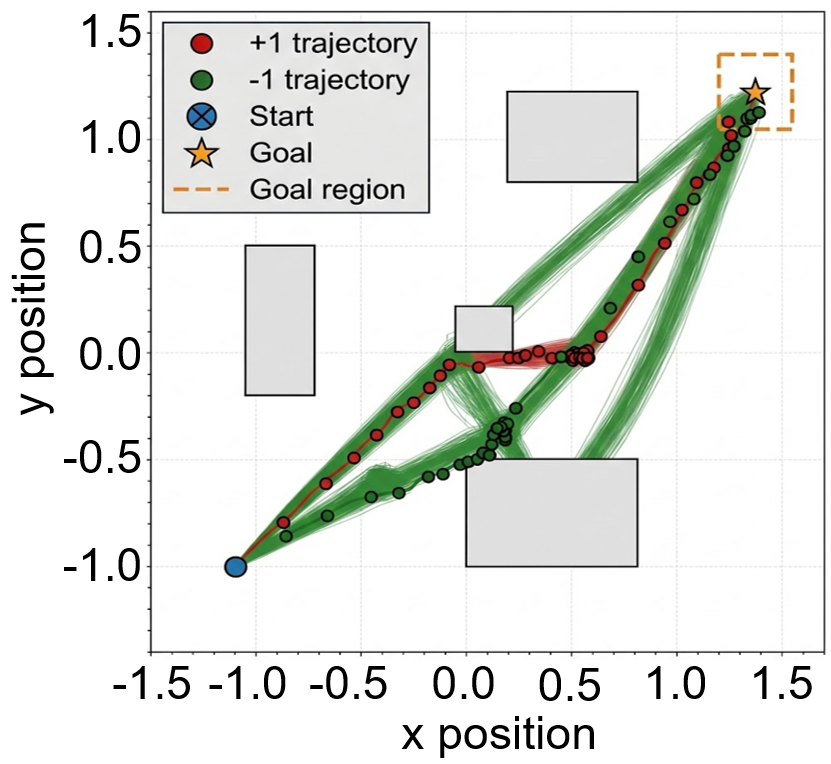}
    \caption{Motion Planning Dataset}
\end{subfigure}
\caption{Schematic Diagram of the Dataset }
\label{fig:trajectory}
\vspace{-1.5em}
\end{figure}

\begin{table*}[t]
\centering
\caption{Classification results of our method and baseline methods on trajectory datasets.}
\label{tab:results}
\small
\begin{tabular}{lccc p{9cm}}
\toprule
Method & MCR for $P_{\mathrm{train}}$ & MCR for $P_{\mathrm{dep}}$ & Time(s) & STL formula \\
\midrule

\multicolumn{5}{c}{\textbf{Naval Dataset}} \\
\midrule
\textbf{Ours} & \textbf{1.25} & \textbf{0.07} & \textbf{16}&
$\boldsymbol{\Diamond_{[25,26]} (x \leq 38.6)\ \wedge\ \Box_{[11,12]} (y \geq 24.1)}$\\

TLINet & 1.25 & 0.05 & 38 &
$\Diamond_{[55,60]}(x<25.89)\wedge\Box_{[0,16]}(y>23.77)$ \\

LSTM/RNN & 0.01& 0.025 & 19 & N/A \\

BCDT & 0.0100 & N/R & 1996 &
$\Diamond_{[28,53]}(x<30.85)\wedge\Box_{[2,26]}(y>21.31)\wedge(x>11.10)$ \\

DT & 0.0195 & N/R & 140 & $\neg(\Diamond_{[38,53]}(x>20.1)\wedge\Diamond_{[12,37]}(x>43.2))
\vee(\Diamond_{[38,53]}(x>20.1)\wedge\neg\Diamond_{[20,59]}(y>32.2))
\vee(\Diamond_{[38,53]}(x>20.1)\wedge\Diamond_{[20,59]}(y>32.2)\wedge\Box_{[14,60]}(y>30.1))$ \\

DAG & 0.0885 & N/R & 996 &
$\Diamond_{[0,33]}(\Box_{[18,23]}(y>19.88)\wedge\Box_{[9,30]}(x<34.08))$ \\

\midrule
\multicolumn{5}{c}{\textbf{VIMA Dataset}} \\
\midrule
\textbf{Ours} & \textbf{8.55} & \textbf{0.015} & \textbf{15} & 
$\boldsymbol{\Box_{[17,19]}((88.62<x<152.21)\wedge(45.47<y<89.73))}$ \\

TLINet & 8.55 & 0.0125 & 45 &
$\Box_{[17,19]}((89.86<x<147.14)\wedge(59.64<y<89.97))$ \\

LSTM/RNN & 8.30 & 0.08 & 23 & N/A \\

\bottomrule
\end{tabular}
\end{table*}

\vspace{-0.8em}
\subsection{Baseline Comparison}
We compare our framework with several baseline classifiers on the considered trajectory datasets: TLINet, a differentiable STL inference model; LSTM/RNN, a recurrent neural-network baseline for sequence classification~\cite{hochreiter1997long}; BCDT, a boosted classification-tree method~\cite{friedman2001greedy,aasi2022classification}; DT, a standard decision tree~\cite{breiman2017classification}; and DAG, a directed-acyclic-graph-based temporal classifier~\cite{platt1999large,kong2016temporal}. These methods serve as reference baselines for classification performance. Neural baselines such as LSTM can be further adapted with additional data, whereas the remaining non-neural baselines do not naturally admit the same gradient-based refinement mechanism as our framework and typically require retraining instead.

Table~\ref{tab:results} report the MCR, training time, and the learned STL formulas for all compared methods. Across the evaluated datasets, models trained only on source-side data degrade under the deployment distribution, which is consistent with covariate shift. Our framework reduces this degradation by refining the core trained STL formula with deployment-side data and remains close to the oracle TLINet trained directly on $P_{\mathrm{dep}}$. Compared with the neural and tree-based baselines, it achieves competitive or better classification performance while retaining an explicit STL formula. The reported training times show that the additional refinement stage remains computationally practical.

\vspace{-0.8em}
\subsection{Ablation Study}
\paragraph{STL Classification}
Table~\ref{tab:results} isolates the effect of the shift-aware refinement stage. Training only on $D_{\mathrm{train}}$ leads to degraded performance under $P_{\mathrm{dep}}$, while incorporating $P_{\mathrm{dep}}$ substantially reduces the MCR across datasets.


\paragraph{Weighted Conformal Prediction under Covariate Shift}

\begin{figure}[t]
    \centering
    \includegraphics[width=\columnwidth]{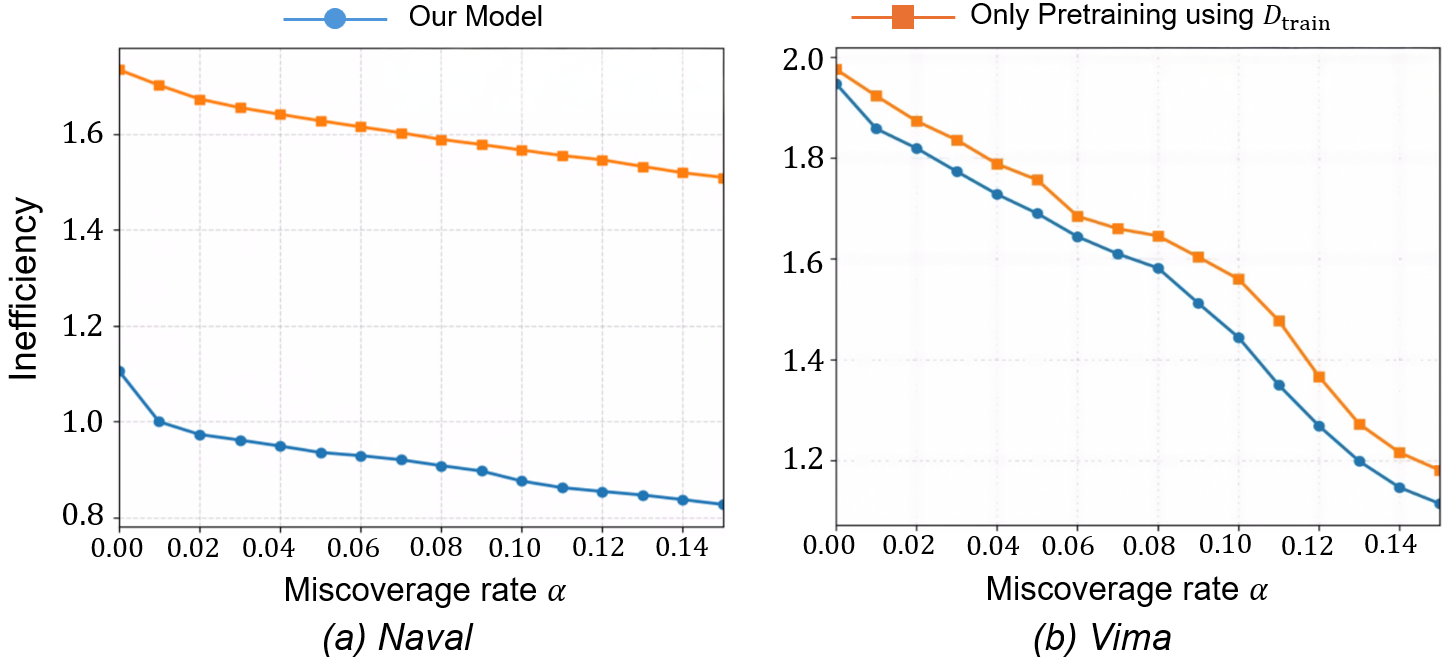}
    \caption{Weighted CP with Covariate Shift }
    \label{fig:stl_comparison}
    \vspace{-1.5em}
\end{figure}

Figure~\ref{fig:stl_comparison} compares weighted CP applied to the refined STL classifier with the same procedure applied to a model trained without shift-aware retraining. In both cases, density-ratio weighting is used during calibration.

On the Naval dataset, the refined model yields substantially smaller prediction set sizes across all miscoverage levels. The gap between the two curves remains consistent, indicating that shift aware retraining improves the alignment between the learned robustness scores and the deployment distribution. As a result, fewer candidate labels are required to maintain the same coverage level. A similar trend is observed on VIMA. Although the difference is less pronounced than in Naval, the refined model consistently achieves lower inefficiency across the range of target miscoverage levels. This suggests that incorporating $D_{\mathrm{dep}}$ during training reduces the impact of marginal distribution mismatch on the calibration score distribution.
Overall, these results indicate that shift-aware refinement enhances the effectiveness of weighted CP by improving the distributional alignment of robustness scores under deployment conditions.

\paragraph{Effect of Shift-aware Refinement on Conformal Prediction}
\begin{figure}[t]
\centering

\begin{subfigure}[t]{\columnwidth}
    \centering
    \includegraphics[width=\columnwidth]{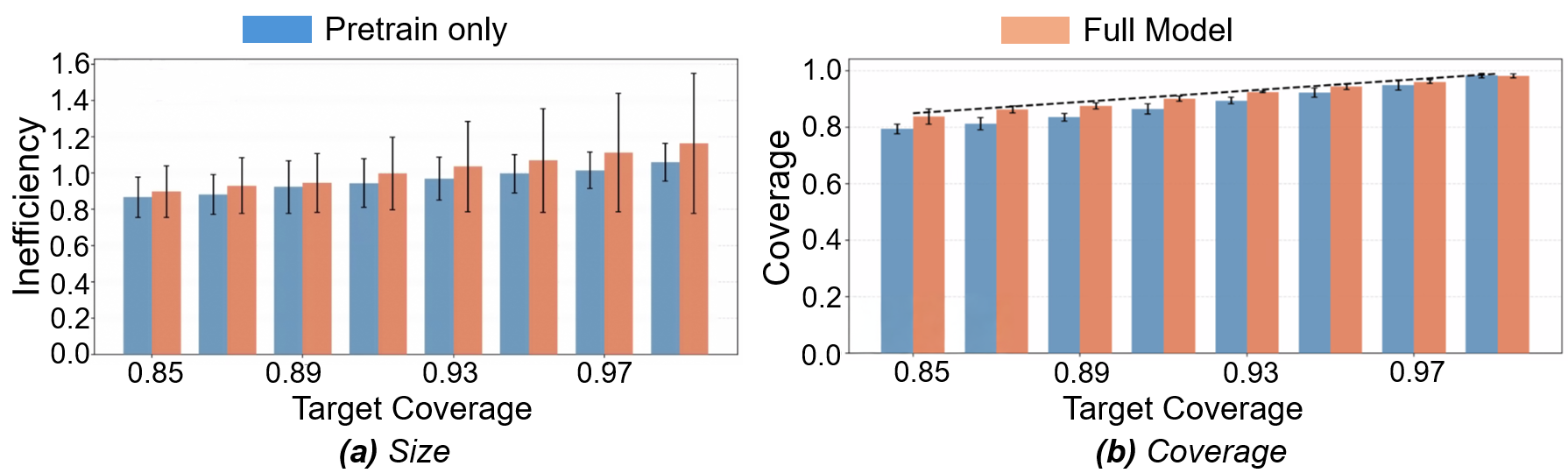}
    \caption{Naval Dataset}
\end{subfigure}

\vspace{0.5em}

\begin{subfigure}[t]{\columnwidth}
    \centering
    \includegraphics[width=\columnwidth]{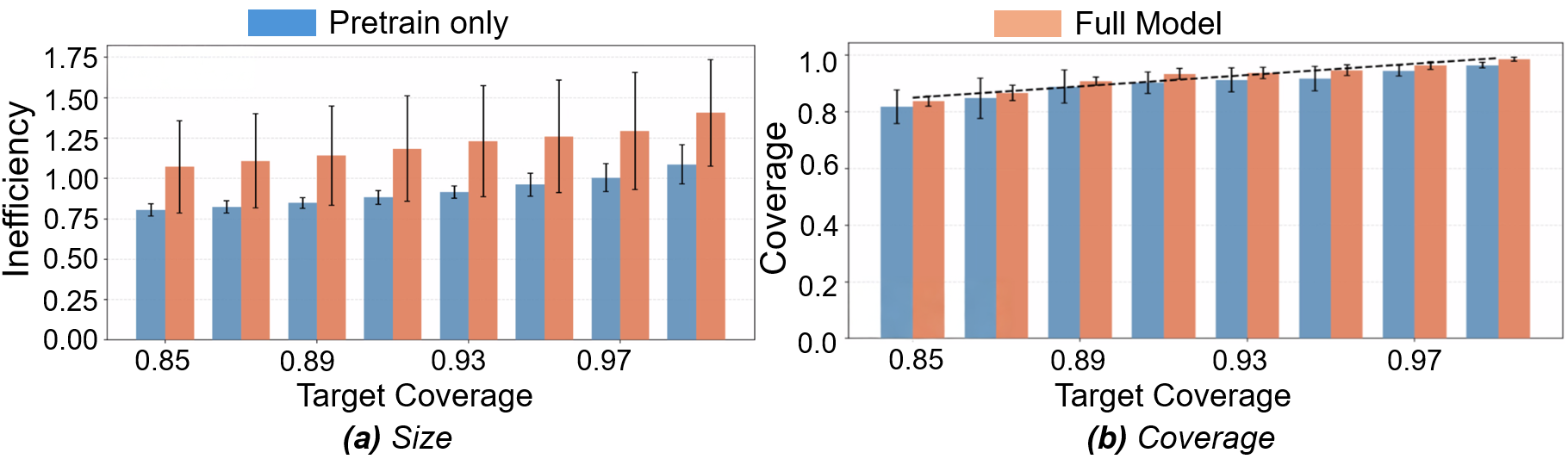}
    \caption{Motion Planning Dataset}
\end{subfigure}

\caption{Comparison across datasets with weighted CP.}
\label{fig:wcp_datasets}
\vspace{-1.0em}
\end{figure}

Figures~\ref{fig:wcp_datasets} further examine how shift-aware refinement affects CP. On both datasets, the refined model yields prediction sets whose average size is closer to one at comparable target coverage levels, indicating better alignment between the learned robustness scores and the deployment distribution. At the same time, empirical coverage remains closer to the target level across settings.


\paragraph{Weighted vs. Standard Conformal Prediction}

\begin{figure}[t]
    \centering
    \includegraphics[width=\columnwidth]{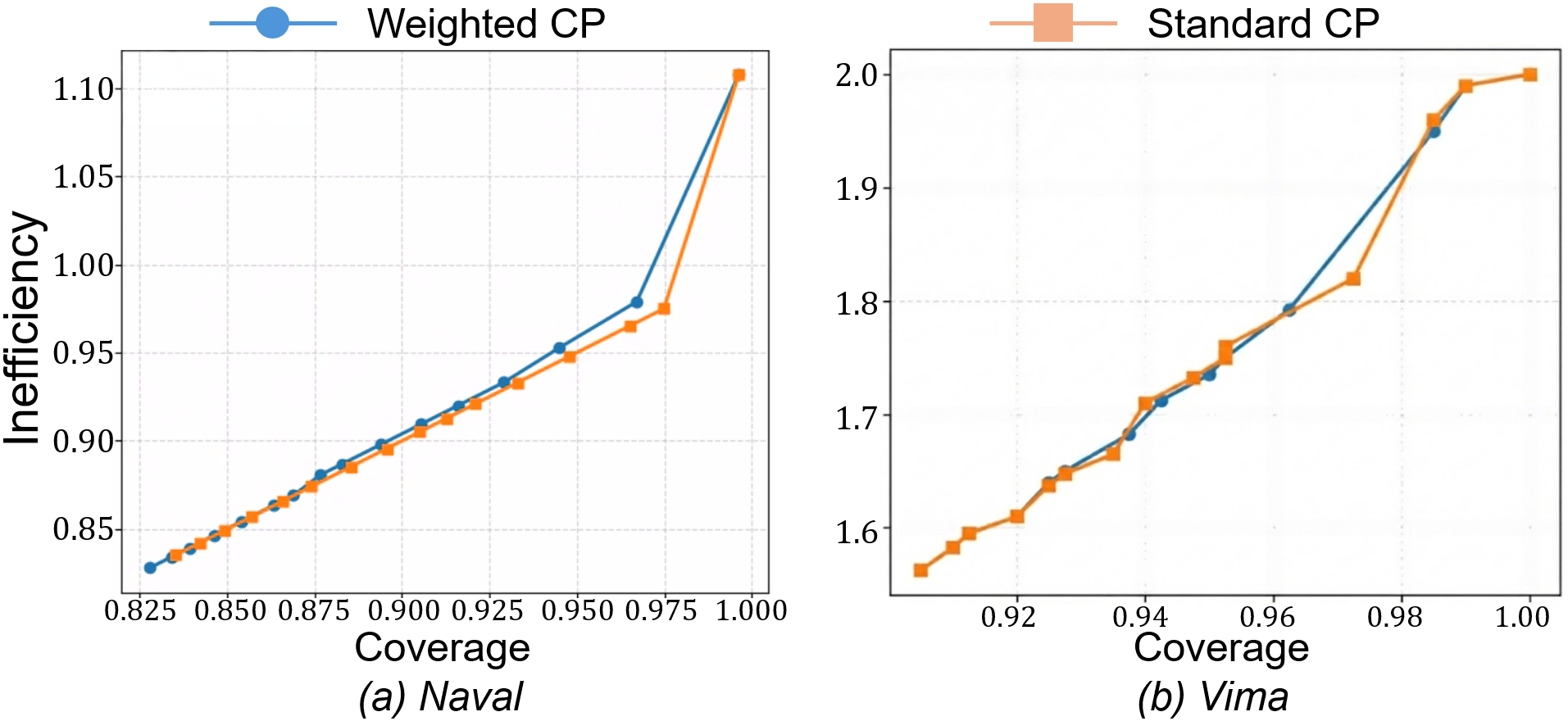}
    \caption{Comparison between weighted CP and Standard CP using our method in different dataset.}
    \label{fig:CP&WCP}
    \vspace{-1.5em}
\end{figure}

Figure~\ref{fig:CP&WCP} compares weighted CP and standard CP in terms of empirical coverage–inefficiency trade-offs on the Naval and VIMA datasets. Across both datasets, the two methods yield closely matched curves, indicating that density-ratio reweighting has only a minor effect on the conformal quantile in the current setting. This suggests that the calibration and test score distributions are already reasonably aligned after refinement.

\vspace{-0.5em}
\section{Conclusion}
We studied STL-based trajectory classification under covariate shift and proposed a shift-aware STL inference framework that refines a learned STL formula using deployment-time data. The method combines TLINet pretraining with a refinement stage that improves the alignment between the learned robustness scores and the deployment distribution. In addition, CP is incorporated to quantify prediction uncertainty while maintaining statistical coverage guarantees. Experimental results on multiple trajectory datasets show that the proposed framework improves classification performance under distribution shift while preserving the interpretability of STL-based decision rules. Future work includes extending the approach to multi-class STL inference and investigating more advanced shift-adaptation strategies.



\bibliography{example_paper}
\bibliographystyle{IEEEtran}




\end{document}